\documentclass[letterpaper, 10 pt, conference]{ieeeconf}  

\IEEEoverridecommandlockouts                              

\usepackage{amsmath} 
\usepackage{amssymb}  

\usepackage{moreverb,url}
\usepackage[colorlinks,bookmarksopen,bookmarksnumbered,citecolor=red,urlcolor=red]{hyperref}
\usepackage{caption}
\usepackage{subcaption}
\usepackage{url}
\usepackage{amssymb}
\usepackage{latexsym}
\usepackage{psfrag}
\usepackage{epsfig}
\usepackage{graphicx}
\usepackage{epstopdf}
\usepackage{color}
\usepackage{tikz}
\usetikzlibrary{calc, arrows.meta}

\newcommand\BibTeX{{\rmfamily B\kern-.05em \textsc{i\kern-.025em b}\kern-.08em
T\kern-.1667em\lower.7ex\hbox{E}\kern-.125emX}}

\newtheorem{theorem}{\textbf{Theorem}}[section]
\newtheorem{lemma}[theorem]{\textbf{Lemma}}

\newtheorem{remark}[theorem]{Remark}

\newtheorem{assumption}[theorem]{\textbf{Assumption}}

\title{\LARGE \bf
Guidance algorithm for smooth trajectory tracking of a fixed wing UAV flying in wind flows.
}

\author{Hector Garcia de Marina$^{1}$, Yuri A. Kapitanyuk$^{2}$, Murat Bronz$^{1}$, Gautier Hattenberger$^{1}$ and Ming Cao$^{2}$
\thanks{This work at ENAC has been supported by the Skycanner project (STAE foundation).}
\thanks{$^{1}$Hector Garcia de Marina, Murat Bronz and Gautier Hattenberger are with the University of Toulouse, Ecole nationale de l'aviation civile (ENAC), 31000 Toulouse, France {\tt\small hgdemarina@ieee.org}.}%
\thanks{$^{2}$Yuri Kapitanyiuk and Ming Cao are with the ENTEG institute at the University of Groningen, the Neherlands.
		{\tt\small \{i.kapitanyuk,m.cao\}@rug.nl}.}%
}

\begin{document}

\maketitle
\thispagestyle{empty}
\pagestyle{empty}

\begin{abstract}
	This paper presents an algorithm for solving the problem of tracking smooth curves by a fixed wing unmanned aerial vehicle travelling with a constant airspeed and under a constant wind disturbance. The algorithm is based on the idea of following a guiding vector field which is constructed from the implicit function that describes the desired (possibly time-varying) trajectory. The output of the algorithm can be directly expressed in terms of the bank angle of the UAV in order to achieve coordinated turns. Furthermore, the algorithm can be tuned offline such that physical constraints of the UAV, e.g. the maximum bank angle, will not be violated in a neighborhood of the desired trajectory. We provide the corresponding theoretical convergence analysis and performance results from actual flights.
\end{abstract}

\section{INTRODUCTION}
The usage of unmanned aerial vehicles (UAVs) in tasks, such as monitoring missions, surveillance or patrolling, has found broad applications. In order to accomplish such missions successfully, it is very often required to track or follow a predetermined path with high accuracy. For example, when performing the aerial mapping for a geographical area of interest, one needs to guarantee that the vehicle will fly over a prescribed trajectory by solving \emph{path-following control problem}. There is no unique approach for addressing this problem for fixed wing UAVs as it has been surveyed in \cite{sujit}. Most popular open-source UAV autopilots (such as Ardupilot \cite{ardu}, Pixhawk \cite{pixhawk} and Paparazzi \cite{papa}) use algorithms that are based on one of the following ideas: Tracking a time-varying reference point \cite{micaelli1993trajectory,soetanto2003adaptive,park2004new} (also known as \emph{carrot-chasing} or \emph{rabbit-chasing}); tracking a vector field \cite{nelson2007vector,frew2007lyapunov}; or minimizing some error signals involving the Euclidean distance to the desired path and other variables \cite{ratnoo2011adaptive}. These algorithms have been shown to be reliable and easy to implement with limited hardware resources; however, they have several limitations. Firstly, they are limited by the necessity of measuring the actual distance between the UAV and the given trajectory. In practice, this restricts the usage of such algorithms to straight lines and circles mostly \cite{nelson2007vector,frew2007lyapunov,ratnoo2011adaptive}; for more generic trajectories they can only provide local stability without information about the region of attraction \cite{park2004new}. Secondly, most of the models using these algorithms do not take into account the wind and they have to address this issue by employing extra controllers in cascade. An integral action can be considered in order to compensate such a disturbance for following a straight line. However, this approach usually fails for generic trajectories, even circles, since when the UAV is following such a path, the wind velocity vector is not fixed with respect to the \emph{body frame}. Techniques such as the estimation of the \emph{sideslip} angle are effective, but typically such results \cite{fossen2015line} are confined to Dubin's path, i.e., straight lines and circles. Thirdly, one can design a generic trajectory matching the physical constraints of the vehicle, e.g., maximum heading rate that determines the bank angle. However, the output from most of the above mentioned algorithms is a heading to be followed by the UAV. This setting point is forwarded usually to another controller in a cascaded fashion, and therefore making it difficult to assert that the physical constraints of the UAV are satisfied when the vehicle is not on the desired path.

The work presented in this paper is an extension of the algorithm given in \cite{YuriCS}. More precisely, the theoretical contribution of our work includes the technique of dealing with wind disturbance when following a generic \emph{sufficiently smooth} 2D path, whereas the practical contributions lie in the adaptation and integration of the algorithm to an actual fixed wing UAV. In particular, the presented algorithm is based on the idea of following a vector field \cite{nelson2007vector,frew2007lyapunov} that converges smoothly to the desired path, where the convergence is global if certain conditions are satisfied. Instead of considering the Euclidean distance, the notion of error is given by the implicit equation of the desired trajectory, making the tracking task much easier to be implemented. Furthermore, this approach makes it possible to deal with the problem of tracking time-varying trajectories or to define a 3D trajectory as the intersection of two surfaces \cite{Wang}. This last feature is desirable for certain practical problems. For example, for the sampling of the atmosphere by UAVs \cite{skyscanner} one can model the boundary of travelling clouds by a 3D \emph{slowly changing} paraboloid parallel to the ground. The desired trajectory for studying the surroundings of such a cloud can be given by the intersection of a plane with this paraboloid.

Note that the guidance vector field is not a novel concept at all and work based on it covering generic trajectories has been presented before \cite{robot}. However, the authors in \cite{robot} have only considered vehicles that can be modeled by fully actuated unit mass points. This kind of model is not suitable for actual fixed wings, where in an optimal or trimmed flight the air-speed must be constant. In our work, we can consider the case when the UAV is flying with a constant air-speed, where our algorithm provides the desired heading-rate for the vehicle. If we consider 2D ground parallel trajectories, then the heading rate can be directly translated to a coordinated turn, i.e. a bank angle that makes the UAV to turn without inducing any acceleration in the lateral axis of the vehicle.  These turns are desirable because they assist the attitude estimators \cite{de2012uav,condomines2015pi} based on the readings of accelerometers, utilizing the observation of the gravity acceleration vector.

The paper is organized as follows. The path following problem under a constant disturbance is set up in Section \ref{sec: prob}. We provide a solution to the problem based on the vector field in Section \ref{sec: ana} derived from Lyapunov stability analysis. We explain the implementation of the algorithm in an actual fixed-wing UAV and present its performance from actual flights in Section \ref{sec: fl}. We finally finish the paper with some conclusions in Section \ref{sec: con}.

The presented algorithm in this paper has been implemented in the popular open-source autopilot system Paparazzi \cite{papa} and it is ready to be used by the general public. 

\section{Problem definition}
\label{sec: prob}
We consider for the fixed wing UAV the following nonholonomic model in 2D
\begin{equation}
\begin{cases}
	\dot p &= sm(\psi) + w \\
	\dot\psi &= u,
\end{cases}
	\label{eq: pdyn}
\end{equation}
where $p\in\mathbb{R}^2$ is the position of the UAV with respect to some inertial navigation frame $\mathcal{O}_N$, $s\in\mathbb{R}^+$ is a constant that can be considered as the \emph{airspeed}, $m = \begin{bmatrix}\cos(\psi) & \sin(\psi)\end{bmatrix}^T$ with $\psi\in(-\pi, \pi]$ being the attitude \emph{yaw} angle, $w\in\mathbb{R}^2$ is a constant\footnote{For the sake of simplicity we consider that $w$ is constant, but we will see after the main result that this requirement can be indeed relaxed.} with respect to $\mathcal{O}_N$ representing the wind and $u$ is the control action that will make the UAV to turn. Here, the UAV is underactuated. We also notice that the course heading $\chi\in(-\pi, \pi]$, i.e. the direction the velocity vector $\dot p$ is pointing at, in general is different from the yaw angle $\psi$ because of the wind.

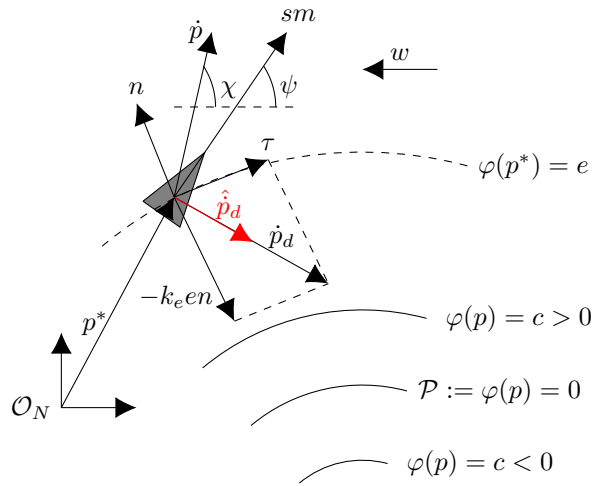
\begin{figure}
	\begin{tikzpicture}
	\draw[fill=gray,cm={cos(-55),-sin(-55),sin(-55),cos(-55),(1.15,0)}](0,1.3)--(0,0.7)--(1,1)--(0,1.3);
	\draw[dashed] (3,-4) ++(75:5.4) arc (75:130:5.4) node[at start, xshift=25]{$\varphi(p^*) = e$};
		\draw (3,-4) ++(75:1.3) arc (75:130:1.3) node[at start, xshift=35]{$\varphi(p) = c < 0$};
		\draw (3,-4) ++(75:2.3) arc (75:130:2.3) node[at start, xshift=35]{$\mathcal{P}:=\varphi(p) = 0$};
		\draw (3,-4) ++(75:3.3) arc (75:130:3.3) node[at start, xshift=35]{$\varphi(p) = c > 0$};;
	\draw[-{Latex[length=8, width=8]}] (-1, -2) -- (0, -2);
		\draw[-{Latex[length=8, width=8]}] (-1, -2) node[left]{$\mathcal{O}_N$} -- (-1, -1);
	\draw[-{Latex[length=8, width=8]}] (-1, -2) -- (0.5,0.8) node[pos=0.3, above]{$p^*$};
	\draw[-{Latex[length=8, width=8]}] (0.5,0.8) -- (1.75, 1.3) node[pos=1, above]{$\tau$};
	\draw[-{Latex[length=8, width=8]}] (0.5,0.8) -- (0, 2.05) node[pos=1, above]{$n$};
	\draw[-{Latex[length=8, width=8]}] (0.5,0.8) -- (1.3, -0.85) node[pos=0.8, left]{$-k_een$};
	\draw[-{Latex[length=8, width=8]}] (0.5,0.8) -- (3.05-0.5, 1.3-0.85-0.8) node[pos=0.7, above]{$\dot p_d$};
		\draw[-{Latex[length=8, width=8]}, color=red] (0.5,0.8) -- (3.05-1.5, 1.3-1.1) node[pos=0.7, above]{$\hat{\dot p}_d$};
	\draw[-{Latex[length=8, width=8]}] (0.5,0.8) -- (2,3) node[pos=1.1]{$sm$};
	\draw[dashed] (0.5, 2) -- (2.1, 2);
	\draw[dashed] (1.3, -0.85) -- (3.05-0.5, 1.3-0.85-0.8);
	\draw[dashed] (1.75, 1.3) -- (3.05-0.5, 1.3-0.85-0.8);
	\draw (1.85,2) arc (0:33:1) node[pos=-0, xshift=5, right, above]{$\psi$};
	\draw (1.05,2) arc (0:33:1) node[pos=-0, xshift=5, right, above]{$\chi$};
	\draw[-{Latex[length=8, width=8]}] (0.5,0.8) -- (1,3) node[pos=1, left]{$\dot p$};
	\draw[-{Latex[length=8, width=8]}] (4,2.5) -- (3,2.5) node[pos=0.5, above]{$w$};
\end{tikzpicture}
	\caption{The direction to be followed by the UAV at the point $p^*$ is given by $\hat{\dot p}_d$. The tangent and normal vectors $\tau$ and $n$ are calculated from $\nabla\varphi(p^*)$. The error \emph{distance} $e$ is calculated as $\varphi(p^*)$, and therefore it is different from the notion of the Euclidean distance in general.}
	\label{fig: ilus}
\end{figure}

Although a fixed wing can fly \emph{backwards} in $\mathcal{O}_N$, i.e. the course heading and the $yaw$ angle differ by $\pi$ radians, for the sake of simplicity in the analysis we consider the following realistic assumption.
\begin{assumption}
	\label{as: 1}
The constant airspeed $s$ is greater than the Euclidean norm of the wind vector $w$, i.e. $s > ||w||$.
\end{assumption}
In fact, it is quite straightforward to check that this assumption is necessary if one wants to reach a generic desired path from almost every initial position.

Consider the desired path $\mathcal{P} \in\mathbb{R}^2$ described by the following implicit equation
\begin{equation}
	\mathcal{P}:= \{p \, : \varphi(p) = 0\},
	\label{eq: P}
\end{equation}
where we assume that the function $\varphi : \mathbb{R}^2 \to \mathbb{R}$ belongs to the $C^2$ space and it is \emph{regular} in a neighborhood of $\mathcal{P}$, i.e.
\begin{equation}
	\nabla\varphi(p) \neq 0, \quad p\in \mathcal{N}_{\mathcal{P}},
	\label{eq: reg}
\end{equation}
where $\mathcal{N}_{\mathcal{P}} := \{p \, : \, |\varphi(p)| \leq c^* \}$ for a constant $c^*\in\mathbb{R}^+$.
The plane $\mathbb{R}^2$ can be covered by the following disjoint sets $\varphi(p) = c \in\mathbb{R}$, where each \emph{level set} is defined for a value, and in particular the \emph{zero level set} $\varphi(p) = 0$ corresponds uniquely to the desired path $\mathcal{P}$. Therefore, we can employ these level sets for the notion of \emph{error distance} between the UAV and $\mathcal{P}$, namely
\begin{equation}
	e(p) := \varphi(p) \in \mathbb{R}. \label{eq: e}
\end{equation}
Note that the error is signed and it differs from the notion of the Euclidean distance.

The main goal is to design a control action $u(p, \dot p, \psi)$ such that $e(t) \to 0$ as $t \to \infty$ and because of (\ref{eq: pdyn}) along with Assumption \ref{as: 1} the UAV will travel over $\mathcal{P}$ with $\dot p(t) \neq 0, \forall t\geq 0$. As will be seen, the control action requires to have available the following states from the UAV: its position and velocity with respect to $\mathcal{O}_N$, from for example a GPS signal and its yaw angle also with respect to $\mathcal{O}_N$, which can be obtained from a well calibrated compass in areas far away from the Earth's poles.
\section{Guidance law design and analysis}
\label{sec: ana}
Let us first introduce some notation. We define by $n(p) := \nabla \varphi(p)$ the normal vector to the curve corresponding to the level set $\varphi(p)$ and the tangent vector $\tau$ at the same point $p$ is given by the rotation
\begin{equation}
	\tau(p) = En(p), \quad E=\begin{bmatrix}0 & 1 \\ -1 & 0\end{bmatrix}. \nonumber
\end{equation}
Note that $E$ will determine in which direction $\mathcal{P}$ will be tracked.

The guidance controller is constructed in two steps. The first one is about constructing a guidance vector field such that once the UAV is tracking it, the vehicle will converge to $\mathcal{P}$. The second step deals with the task of steering the UAV in order to converge to such a guiding vector field.

\subsection{Vector field design}
The main idea to construct a guiding vector field pointing at $\mathcal{P}$ is based on decreasing the norm of (\ref{eq: e}). Consider the following Lyapunov function
\begin{equation}
	V_1(p) = \frac{1}{2}||e(p)||^2,
\end{equation}
whose time derivative along (\ref{eq: pdyn}) is given by
\begin{align}
	\frac{\mathrm{d}V_1}{\mathrm{dt}} = e n^T \dot p.
\end{align}
Consider the following desired velocity vector
\begin{equation}
	\dot p_d(p) := \tau(p) - k_e e(p)n(p),
	\label{eq: gvf}
\end{equation}
where $k_e\in\mathbb{R}^+$ is a gain that will tune how \emph{aggressive}
the vector field is. It is clear that
\begin{equation}
	e n^T \dot p_d = e n^T\tau - e^2k_e||n||^2 =  - e^2k_e||n||^2 \leq 0,
\end{equation}
is decreasing if and only if $e \neq 0$ for $p\in \mathcal{N}_\mathcal{P}$. Note that since $\tau$ is perpendicular to $n$, once the UAV is over $\mathcal{P}$ then the vehicle will track the direction given by only $\tau$, i.e. the tangent to the desired path. Therefore, we define (\ref{eq: gvf}) as the \emph{guidance vector field} to be followed. In particular, the UAV has to track the orientation of the unit vector calculated from (\ref{eq: gvf}), i.e. the desired course heading $\chi_d(p)$.

\begin{remark}
It is now clear the role of Assumption \ref{as: 1}. If $s < ||w||$ then it is not possible to travel in the direction against the wind, and therefore being impossible to track the guidance vector field (\ref{eq: gvf}) in general.
\end{remark}

Let us define $\hat{x}:= \frac{x}{||x||}$ for $x\in\mathbb{R}^n$. Now we are going to calculate what the desired course heading rate $\dot\chi_d(\dot p, p)$ 
is once the UAV is tracking correctly the guidance vector field given in (\ref{eq: gvf}), i.e. what the course heading rate is such that the set ${\mathcal{G}} :=  \{\dot p \, : \, \hat{\dot p} = \hat{\dot p}_d\}$ is invariant.

The time derivative of the unit vector defining the desired heading is given by
\begin{align}
	\frac{\mathrm{d}\hat{\dot p}_d}{\mathrm{dt}} &= (I - \hat{\dot p}_d \hat{\dot p}_d^T)\frac{\ddot p_d}{||\dot p_d||} = (\hat{\dot p}_d^TE)^T(\hat{\dot p}_d^TE)\frac{\ddot p_d}{||\dot p_d||} \nonumber \\
&= -E\hat{\dot p}_d \, \hat{\dot p}_d^T E \frac{\ddot p_d}{||\dot p_d||},
\label{eq: ph_d}
\end{align}
where $I$ is the identity matrix with the appropriate dimensions and from (\ref{eq: gvf}) we derive
\begin{align}
	\ddot p_d &= \frac{\mathrm{d}}{\mathrm{dt}}(E-k_ee)n \nonumber \\
	&= (E -k_ee)H(\varphi(p))\dot p - k_en^T\dot p n,
	\label{eq: pd_dd}
\end{align}
where $H(\cdot)$ is the Hessian operator, establishing then the condition of $\varphi(p)$ being $C^2$. Physically it means that the UAV in order to track $\mathcal{P}$ needs to know how the curvature of the desired trajectory evolves. 

Now we derive the expression of the desired heading rate $\dot\chi_d$ once $\dot p \in\mathcal{G}$. Since $||\hat {\dot p}_d||^2=1$, we have that\begin{equation}
	\frac{1}{2}\frac{\mathrm{d}||\hat {\dot p}_d||^2}{\mathrm{dt}} = \hat {\dot p}_d^T \frac{\mathrm{d}\hat{\dot p}_d}{\mathrm{dt}} = 0,
\end{equation}
hence the infinitesimal rotation of $\hat{\dot p}_d$ can be described by
\begin{equation}
	\frac{\mathrm{d}\hat {\dot p}_d}{\mathrm{dt}} = -\dot\chi_d E\hat{\dot p}_d,
	\label{eq: w}
\end{equation}
which is perpendicular to $\hat {\dot p}_d$ and the angular speed is determined by $\dot\chi_d$. Working out further (\ref{eq: w}) we identify that
\begin{align}
	E^T \frac{\mathrm{d}\hat {\dot p}_d}{\mathrm{dt}} &= -\dot\chi_d \hat {\dot p}_d \nonumber \\
	\hat {\dot p}_d^T E^T \frac{\mathrm{d}\hat{\dot p}_d}{\mathrm{dt}} &= -\dot\chi_d \nonumber \\
	\dot\chi_d &= \left(\frac{\mathrm{d}\hat {\dot p}_d}{\mathrm{dt}}\right)^T E \,\hat{\dot p}_d,
	\label{eq: ud}
\end{align}
therefore by tracking back (\ref{eq: pd_dd}), (\ref{eq: ph_d}) and (\ref{eq: gvf}) for the angular velocity (\ref{eq: ud}) we notice that the desired course heading rate $\dot\chi_d$ in order to keep $\mathcal{G}$ invariant can be computed by only sensing the ground velocity $\dot p$ and position $p$ of the UAV. We summarize such observations into the following Lemma.
\begin{lemma}
	\label{lem: 1}
	The set ${\mathcal{G}} :=  \{\dot p \, : \, \hat{\dot p} = \hat{\dot p}_d\}$, with $\dot p_d$ being the guidance vector field in (\ref{eq: gvf}), is invariant for the following course heading rate
\begin{align}
	&\dot\chi(p,\dot p) = \nonumber \\
	&-\left(E\hat{\dot p}_d \, \hat{\dot p}_d^T E \left(\left(E -k_ee\right)H(\varphi(p))\dot p - k_en^T\dot p n\right)\right)^T  E\frac{\dot p_d}{||\dot p_d||^2}, \label{eq: wlem}
\end{align}
that only depends on $\dot p$ and $p$.
\end{lemma}
\begin{proof}
	By direct inspection of (\ref{eq: wlem}) it is clear that it only depends on $\dot p$ and $p$. Consider that $p\in\mathcal{G}$ with $\dot p_d$ as in (\ref{eq: gvf}). Then in order to keep $\mathcal{G}$ invariant we have to satisfy the right hand side of (\ref{eq: ph_d}) for the time derivative of $\hat{\dot p}$. The angular velocity (\ref{eq: wlem}) is the substitution of (\ref{eq: pd_dd}) and (\ref{eq: ph_d}) into (\ref{eq: ud}), which determines the course heading rate $\dot\chi$ in order to keep $\mathcal{G}$ invariant.
\end{proof}
\begin{remark}
	Notice that the yaw rate $\dot\psi$ and the course heading rate $\dot\chi$ are different concepts and quantities. However, according to (\ref{eq: pdyn}) by actuating over the yaw, we are also actuating over the course heading. The aim of the next section is to show how to design $u$ in (\ref{eq: pdyn}) such that the UAV is following the appropriated course heading rate.
\end{remark}

\subsection{Converging to the guidance vector field}
Now we are going to present how to make the UAV to converge to the guidance vector field defined in (\ref{eq: gvf}). The ground velocity $\dot p$ can be trivially decomposed as $\dot p = ||\dot p|| \hat{\dot p}$. Now consider the following Lyapunov function
\begin{equation}
	V_2(\eta) = 1 - \hat {\dot p}^T \hat{\dot p}_d,
	\label{eq: V2}
\end{equation}
where $\eta\in(-\pi, \pi]$ is the angle between the two unit vectors $\hat {\dot p}$ and $\hat{\dot p}_d$. It is clear that $V_2(\eta) = 0 \iff \eta = 0$, i.e. the ground velocity of the UAV is aligned with the vector field. Now we are ready for our main result.
\begin{theorem}
	\label{th: 1}
	Consider a desired trajectory $\mathcal{P}$ as in (\ref{eq: P}) such that $\varphi(p)$ is $C^2$ and satisfies (\ref{eq: reg}). Assume that the UAV is modeled by (\ref{eq: pdyn}) under Assumption \ref{as: 1} and the vehicle can measure its ground velocity $\dot p$, position $p$ and yaw angle $\psi$ with respect to some navigation frame $\mathcal{O}_N$. Then the control action 
\begin{equation}
	u(\dot p, p, \psi) = \dot\psi = \frac{||\dot p||}{s\cos{\beta}}\left (\dot\chi_d(\dot p, p)
	+ k_d \hat{\dot p}^TE\hat{\dot p}_d\right), \label{eq: u}
\end{equation}
	where $\beta = \operatorname{arccos}\left({\hat{\dot p}^T m(\psi)}\right)$ is the \emph{sideslip} angle, $k_d\in\mathbb{R}^+$ determines how fast the UAV converges to the guidance vector field and $\dot\chi_d$ is given in Lemma \ref{lem: 1}, guides the UAV (at least locally) to converge asymptotically to travel over $\mathcal{P}$ for all the initial conditions $p(0)\in \mathcal{N}_c \subset \mathcal{N}_{\mathcal{P}}$, where $\mathcal{N}_c$ is as $\mathcal{N}_{\mathcal{P}}$ in (\ref{eq: reg}) but with a constant $0 \leq c < c^*$.
\end{theorem}
\begin{proof}
	We need to show that the unit velocity vector $\hat {\dot p}$ of the UAV converges asymptotically to the unit velocity vector $\hat {\dot p}_d$ given by the vector field (\ref{eq: gvf}), i.e. $\chi(p, \dot p, \psi) - \chi_d(p, \dot p, \psi) \to 0$ as $t \to \infty$. Consider that $p(0)\in\mathcal{N}_c$ and take the time derivative of the Lyapunov function (\ref{eq: V2})
\begin{align}
	\frac{\mathrm{d}V_2}{\mathrm{dt}} &= -\hat{\dot p}_d^T \left(\frac{\mathrm{d}}{\mathrm{dt}} \hat {\dot p}\right) - \hat {\dot p}^T \left(\frac{\mathrm{d}}{\mathrm{dt}}\hat{\dot p}_d\right). \label{eq: V2d}
\end{align}
We now work out the first time derivative term on the right hand side of (\ref{eq: V2d}) since the second term has been calculated in (\ref{eq: w}):
\begin{align}
\frac{\mathrm{d}}{\mathrm{dt}} \hat {\dot p} &= -\frac{1}{||\dot p||}E
	\hat{\dot p}\hat{\dot p}^TE\frac{\mathrm{d}}{\mathrm{dt}}\dot p\nonumber \\
	&= -\frac{s}{||\dot p||}E\hat{\dot p}\hat{\dot p}^TE \frac{\mathrm{d}}{\mathrm{dt}} m(\psi) \nonumber \\
	&= \frac{s}{||\dot p||}E\hat{\dot p}\hat{\dot p}^TE \dot\psi E m(\psi) \nonumber \\
	&= -\frac{s\dot\psi}{||\dot p||}\left(\hat{\dot p}^T m(\psi)\right)E \hat{\dot p}. \label{eq: tpm}
\end{align}
	We now substitute (\ref{eq: tpm}) and (\ref{eq: w}) into (\ref{eq: V2d}) and we arrive at
\begin{align}
	\frac{\mathrm{d}V_2}{\mathrm{dt}} &= \frac{s\dot\psi}{||\dot p||}\left(\hat{\dot p}^T m(\psi)\right)\hat{\dot p}_d^TE \hat{\dot p} - \dot\chi_d\hat{\dot p}_d^T E\hat{\dot p} \nonumber \\
	&= \left(\frac{s\dot\psi}{||\dot p||}\cos{\beta} - \dot\chi_d \right)\hat{\dot p}_d^T E\hat{\dot p}.
\end{align}
By choosing
\begin{equation}
	u(\dot p, p, \psi) = \dot\psi = \frac{||\dot p||}{s\cos{\beta}}\left (\dot\chi_d + k_d \hat{\dot p}^TE\hat{\dot p}_d\right),
	\label{eq: controlac}
\end{equation}
we have that
\begin{equation}
	\frac{\mathrm{d}V_2}{\mathrm{dt}} = -k_d (\hat{\dot p}^TE\hat{\dot p}_d)^2 \leq 0,
	\label{eq: bajo}
\end{equation}
	which is non-increasing in the (compact) set $\mathcal{N}_{\mathcal{P}}$. The constant $k_d$ has to be big enough such that the UAV does not leave the set $\mathcal{N}_{\mathcal{P}}$ once the vehicle starts in $\mathcal{N}_c$, i.e. we need to align the UAV with the vector field \emph{as soon as possible} in order to satisfy (\ref{eq: reg}) for all $t\geq 0$. The (worst case) calculation of $k_d$ is a strictly geometrical and kinematic task that depends on $s$, $w$, $\psi$, $c$ and $c^*$. Because of (\ref{eq: bajo}), the condition (\ref{eq: reg}) and Assumption (\ref{as: 1}), by invoking the LaSalle's invariance principle, we conclude that $\chi\left(p(t), \dot p(t), \psi(t)\right) - \chi_d\left(p(t), \dot p(t), \psi(t)\right) \to 0$, or equivalently $ \hat{\dot p}(t) \to \mathcal{G}$, as $t \to \infty$, implying that $p(t) \to \mathcal{P}$ as $t \to \infty$ with the UAV travelling over the desired trajectory $\mathcal{P}$ with $\dot p(t)\neq 0, \forall t$.
\end{proof}
\begin{remark}
	Note that for some trajectories, such as straight lines, we have that $\mathcal{N}_{\mathcal{P}} = \mathbb{R}^2$, and therefore the convergence in Theorem \ref{th: 1} is global. For other trajectories where the set of critical points $\nabla\phi(p) = 0$ is bounded and does not intercept $\mathcal{P}$, e.g. the center of an ellipse, a more precise condition for $k_d$ can be given in order to be more specific about $\mathcal{N}_{\mathcal{P}}$ and $\mathcal{N}_c$ \cite{YuriCS}.
\end{remark}
\begin{remark}
We also note that if the yaw angle $\psi$ is not accessible or reliable, e.g. if the UAV is close to the Earth's poles, one can replace it by measuring the sideslip angle $\beta$ and still employing the results from Theorem \ref{th: 1}.
\end{remark}
\begin{remark}
	One can check that a non-constant positive airspeed $s(t)$, but satisfying Assumption \ref{as: 1}, will not change the convergence results in Theorem \ref{th: 1}. In fact the condition of having a constant wind $w$ can also be relaxed by just considering that $s(t) > \operatorname{sup}\{||w(t)||\}, \forall t \geq 0$, i.e. we exclude situations where the UAV stops or flies \emph{backwards} with respect to the ground.
\end{remark}

\section{Implementation and flight performance}
\label{sec: fl}
In this Section we are going to discuss several practical issues in order to implement the guidance vector field (\ref{eq: gvf}) to an actual fixed wing in the opensource project Paparazzi \cite{papa}. We conclude the section by showing the performance of actual flights employing the results\footnote{For more details about the implementation of the algorithm and further experimental results we refer to the website https://wiki.paparazziuav.org/wiki/Module/guidance\_vector\_field} in Theorem \ref{th: 1}.

\subsection{Gain tuning}
We are going to show that the gains $k_e$ in the guidance vector field (\ref{eq: gvf}) and $k_d$ in the control action (\ref{eq: u}) can be designed in order to satisfy the physical constraint given by the maximum bank angle $\phi^*$ of the UAV. 

If we consider that the flight path angle is zero, i.e. the UAV is keeping its altitude, and $s >> ||w||$, i.e. we have a small sideslip $\beta$, the yaw rate $\dot\psi$ can be well approximated by the following expression \cite{stevens2015aircraft}
\begin{equation}
	\dot\psi = \frac{g\tan\phi \cos\theta}{s},
	\label{eq: turn}
\end{equation}
where $g$ is the gravity acceleration and $\phi$ and $\theta$ are the roll and pitch attitude angles, respectively, of the UAV. The expression (\ref{eq: turn}) is also known as the condition for a \emph{coordinated turn}. In such a case the UAV does not experience any acceleration in its lateral body axes. This is desirable since many of the attitude estimation algorithms employed in projects like Paparazzi, such as \cite{de2012uav,condomines2015pi}, are based on the observation of gravity. We also consider that in a trimmed flight, the pitch angle $\theta$ remains constant and usually is close to zero.

From (\ref{eq: turn}) it is clear that we have to satisfy
\begin{equation}
	|\phi^*| \leq \arctan \frac{s\, u(\dot p, p, \psi)}{g \cos\theta}.
	\label{eq: constr}
\end{equation}
If one is interested in visiting a certain area with the UAV, then it should restrict the desired trajectory $\mathcal{P}$ such that (\ref{eq: constr}) is satisfied under the worst case condition for $||\dot p||$ (determined by the expected wind speed) in (\ref{eq: u}) with $\eta = 0$. Our algorithm covers the popular splines \cite{gill2015spline} for trajectory generation in order to satisfy constraints such as (\ref{eq: constr}).

A conservative value for $k_d$ in order to satisfy (\ref{eq: constr}) can be calculated by considering $\hat{\dot p}^TE\hat{\dot p}_d = \pm 1$ in (\ref{eq: controlac}). However, one needs to have in mind that $k_d$ should be sufficiently big (in the transient of the UAV converging to the guidance vector field) according to Theorem \ref{th: 1} for keeping $p(t)\in\mathcal{N}_{\mathcal{P}}$. Finally, $k_e$ directly influences how smooth (compromised by how fast) the convergence of the guidance vector field to $\mathcal{P}$ is and it can be chosen arbitrarily small. Therefore, one can calculate beforehand the values of $k_e$ and $k_d$ such that (\ref{eq: constr}) is satisfied in $\mathcal{N}_{\mathcal{P}}$.

\subsection{Experimental platform}
We have tested the validity of Theorem \ref{th: 1} in our fixed wing UAV shown in Figure \ref{fig: jump} called \emph{Jumper}. It is about $450$grams of weight, $70$cm of wingspan, actuated by two elevons and one motor. The electronics include a battery that allows about $30$ minutes of autonomoy at the nominal flight, which corresponds to about an airspeed of $s = 11$m/s. The vehicle has a high maneuverability and we have set in the autopilot a saturation of $|\phi^*| = 45$ degrees to the roll angle. The chosen board for the autopilot is the Apogee \cite{papa}, supported by Paparazzi, which includes the usual sensors of three axis gyros, accelerometers, magnetometers and a GPS. Therefore we can measure $p, \dot p$ and $\psi$ as required in Theorem \ref{th: 1}. The microcontroller on board is a STMicroelectronics STM32F4. Although there is a logging system on board, the vehicle counts with a serial radio link in order to monitor its status from the ground. The algorithm in Theorem \ref{th: 1} has been programmed as a (guidance) module in Paparazzi and it can be combined or integrated easily with other modules in the system. In particular, we have set the periodic frequency of the \emph{guidance vector field module} to $60$Hz. The source code can be checked online at the Paparazzi repository, where the implementation of the guidance algorithm is independent of the trajectory. This allows other users to specify their own trajectories by only defining the implicit $\varphi(p), \nabla\varphi(p)$ and $H(\varphi(p))$ in C-code. In addition, Paparazzi allows easily to change parameters on flight, and different values for $k_e$ and $k_d$ can be set on-the-fly from the ground station.

\begin{figure}
\centering
\includegraphics[width=1\columnwidth]{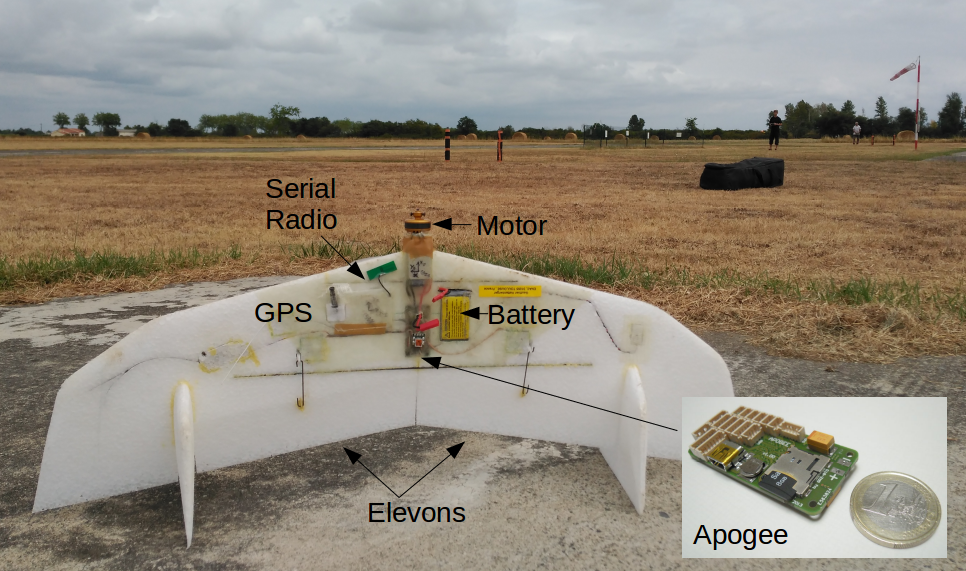}
	\caption{UAV Jumper at the aero model club of Eole at Muret (Toulouse) before the experiment's flight in a pre-storm day.}
\label{fig: jump}
\end{figure}

\subsection{Flight experiments}
The flights have taken place at the aero model club of Eole at Muret, close to the city of Toulouse in France. We performed the flights on the 18th of August, 2016 between the 14:00 and 18:00 hours local time. The wind according to the weather service of MeteoFrance was about $5$m/s blowing from the east with gusts of about $10$m/s. Therefore Jumper with a nominal airspeed of $s=11$m/s satisfies Assumption \ref{as: 1}.

As a benchmark we consider different ellipses as desired $\mathcal{P}$, namely
\begin{align}
	\varphi(p) &= \left(\frac{(p_x - h_x)\cos\alpha - (p_y - h_y)\sin\alpha}{a}\right)^2 \nonumber \\ &+ \left(\frac{(p_x - h_x)\sin\alpha + (p_y - h_y)\cos\alpha}{b}\right)^2 - 1,
\label{eq: elli}
\end{align}
where $h = \begin{bmatrix}h_x & h_y\end{bmatrix}^T$ is the center of the ellipse with respect to $\mathcal{O}_N$, $\alpha$ is the rotation angle of the ellipse with respect to the horizontal axis of $\mathcal{O}_N$ and $a$ and $b$ are the lengths of the ellipse's axis. Note that only for $p = h$ we have that $\nabla\varphi(p) = 0$, and therefore (\ref{eq: elli}) satisfies (\ref{eq: reg}) for some $c^* > 0$.

The autopilot allows to have a fully automated flight, from the take-off until the landing. We show in Figure \ref{fig: xy} one of the tested ellipses corresponding to $a = 50, b = 75$ meters and $\alpha = -15$ degrees, together with the described Jumper's trajectory. We have designed $k_e = 0.4$ and $k_d = 1$ such that $\phi$ (without wind) is less than $45$ degrees for $c^* \leq 6$ in $\mathcal{N}_{\mathcal{P}}$. We describe the experiment in more detail in Figures \ref{fig: xy}-\ref{fig: roll}.

\begin{figure}[tb]
\centering
\includegraphics[width=0.92\columnwidth]{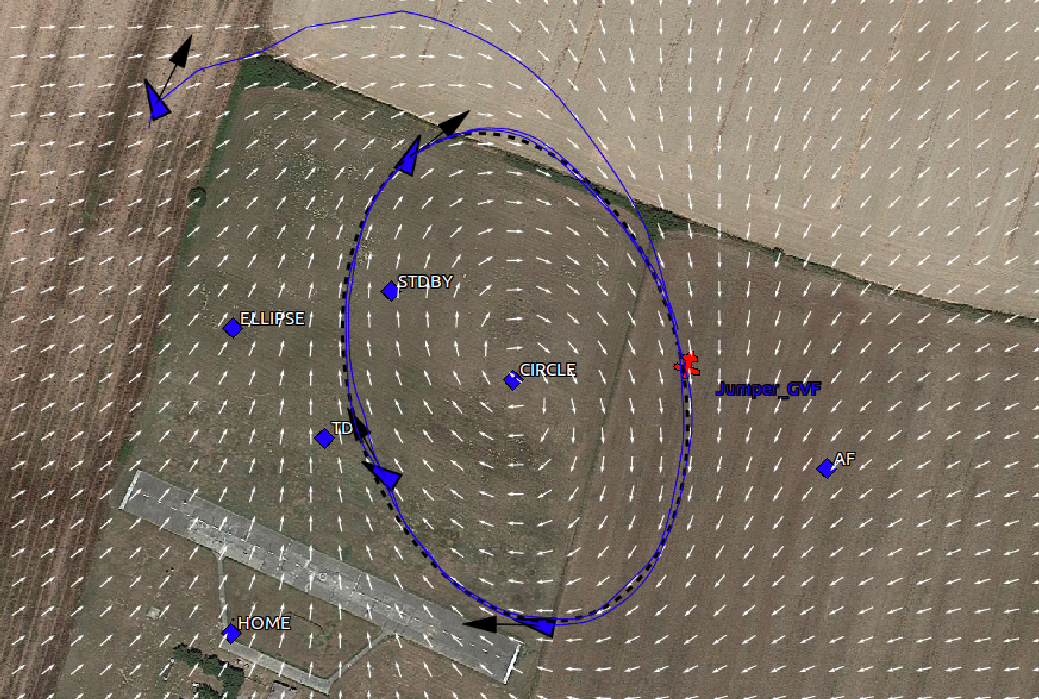}
	\caption{Screenshot from the Paparazzi's ground control station. From the flight log we have drawn on top several positions of the UAV (blue triangles). The yaw $\psi$ is represented by the orientation of the triangle and the black arrows stand for the course heading. The desired ellipse has been marked with a black dashed line. We have marked the last position of the UAV, after two turns to the ellipse, in red color. The vector field is represented in white color and the blue line is the actual trajectory.}
\label{fig: xy}
\end{figure}

\begin{figure}
\centering
\includegraphics[width=1\columnwidth]{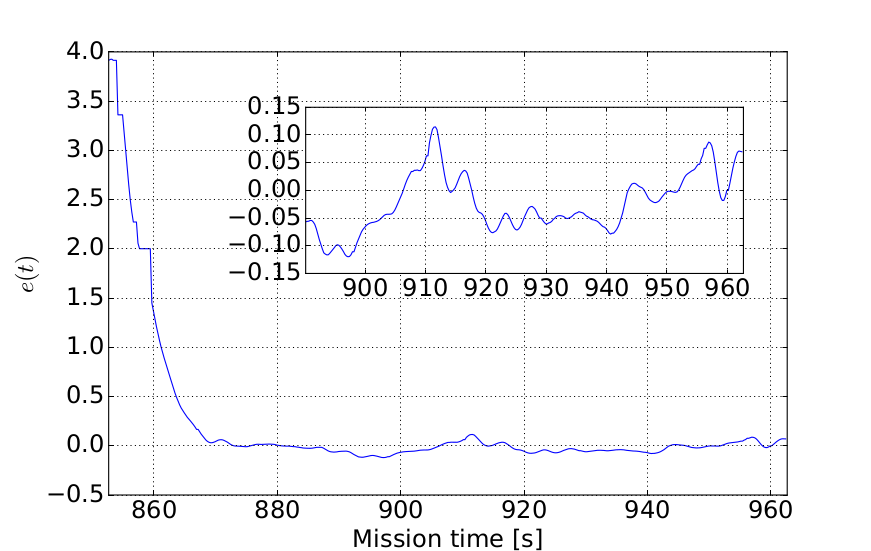}
	\caption{Evolution of the adimensional error signal calculated as in (\ref{eq: elli}). Note that this signal is different from the notion of Euclidean distance.}
\label{fig: e}
\end{figure}

\begin{figure}
\centering
\includegraphics[width=1\columnwidth]{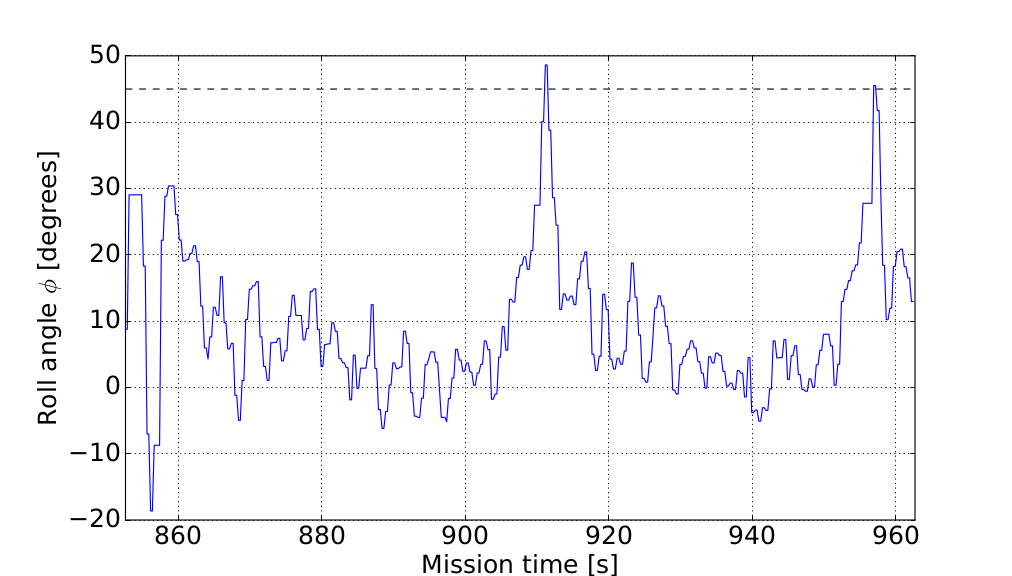}
\caption{Roll angle signal. The black dashed line is the constraint $\phi^*$. 
	It is passed when the UAV has tail wind (maximum ground speed) and crosses
	the minor axis of the ellipse (maximum turning rate for $\mathcal{P}$). The tuning of the gains $k_e$ and $k_d$ was done considering almost no wind, but in the experiment the wind is almost the $50\%$ of the airspeed.}
\label{fig: roll}
\end{figure}


\section{Conclusions}
\label{sec: con}
This paper has presented an algorithm for making fixed wing UAVs following smooth trajectories under the presence of wind. The guidance strategy is based on tracking a vector field generated from the implicit form of the desired trajectory.
The simplicity of this algorithm allows its implementation in small embedded systems as the Apogee autopilot. The algorithm has been implemented in Paparazzi as an independent module and does not depend on the desired trajectory, allowing other users to employ the algorithm by just codifying the implicit equation of the trajectory, along with its gradient and Hessian. 

We are currently extending and testing the results of this paper for formation flying control by employing the different level sets of a desired trajectory as an input for consensus algorithms and combining the recent findings in \cite{de2016distributed,garcia2015controlling} for controlling rigid formations.

\bibliographystyle{IEEEtran}
\bibliography{hector_ref.bib}

\end{document}